%% file: rootrml.tex
\documentclass{article}

\def\noheaderplainsetup{

\topmargin=0pt \headheight=0pt \headsep=0pt  \oddsidemargin=0pt \evensidemargin=0pt  \textheight=9.1truein \textwidth=6.5truein}   

\noheaderplainsetup

\usepackage{amsfonts}

\begin{document}

\input clmacros1
\newcommand{\Dscr}{\cal D}
 \newcommand{\propel}{\mbox{\bf CL1}}
\newcommand{\propelw}{\mbox{\bf CL1$^\Omega$}}
\newcommand{\propeltw}{\mbox{\bf CL12}}
\newcommand{\colw}{\propelw}
\newenvironment{exmple}{
 \begingroup \begin{tabbing} \hspace{2em}\= \hspace{3em}\= \hspace{3em}\=
\hspace{3em}\= \hspace{3em}\= \hspace{3em}\= \kill}{
 \end{tabbing}\endgroup}
\newenvironment{example2}{
 \begingroup \begin{tabbing} \hspace{8em}\= \hspace{2em}\= \hspace{2em}\=
\hspace{10em}\= \hspace{2em}\= \hspace{2em}\= \hspace{2em}\= \kill}{
 \end{tabbing}\endgroup}

\title{Towards Distributed Logic Programming based on Computability Logic}
\author{Keehang Kwon  \\ 
  {\small Department of Computing Sciences, DongA University, South Korea.
 khkwon@dau.ac.kr}}
\date{}
\maketitle


{\bf Abstract:}

{\em Computability logic} (CoL) is a powerful computational model which views 
computational problems as games played by a machine and its environment. 
In this paper,  we show that CoL
naturally supports   multiagent
programming models with distributed control. 
To be specific, we discuss 
a distributed logic programming  model based on CoL (CL1 to be exact), which we call 
\colw.
The key feature of this model is that it supports $dynamic/evolving$  
knowledgebase of an agent. This model turns out to be 
a promising approach to reaching both general AI and future computing model.



\section{Introduction}\label{sintr}

{\em Computability logic} (CoL) \cite{Jap0}-\cite{JapCL12}, is an
elegant theory of (multi-)agent computability. In CoL, computational problems are seen as games between a machine and its environment and logical operators stand for operations on games.
It understands interaction among agents in its most general --- game-based --- sense. On the other hand, other formalisms such as situation calculus
 appear to be too rudimentary to represent complex interactions among agents.
 In particular,  CoL supports 
 
 \[  query/knowledge\ duality \]
  (or we call it
  `querying knowledge'):
 what is a query for one agent becomes  new knowledge for another agent.
 This duality leads to  $evolving$ knowledgebase and has many
 attractive features such as local namespace.
  Note that traditional agent/object-oriented approaches \cite{LCF} fail to support this duality.
Therefore,  CoL  provides a promising basis
for  multiagent programming.

In this paper, we discuss a distributed agent programming model based on CoL, which
can also be seen as a distributed logic programming model with distributed processing.
In CoL, the environment is assumed to be an unpredictable, capricious user. In 
contrast, we make it possible for an environment to be specified as a machine with  
determined, algorithmic behavior.

We assume the following in our model:

\begin{itemize}

\item Each agent corresponds  to a memory location or a 
web site with a URL. An agent's  knowledgebase(KB) is stored in its location.

\item  Agents are initially inactive. An  inactive agent  becomes activated  when another agent  invokes a query  for the former.

\item  Our model supports query/knowledge duality and querying knowledge.
  That is, knowledge of an agent can be obtained from another agent by invoking
  queries to the latter.
\end{itemize}

To make things simple, we choose \propel -- the most basic fragment of CoL -- as our target language.  \propel\ is  obtained  by adding to  classical propositional logic two additional choice operators: disjunction ($\add$) and conjunction ($\adc$) operators. The choice disjunction  
$\add$ models decision steps by the machine.
The choice conjunction  $\adc$ models decision steps by the environment. For example, $green \add red$ is a game
where the machine must choose either $green$ or $red$, while $green \adc red$ is a game
where the environment must choose either $green$ or $red$.  In the former, if the machine chooses $green(red)$,
then we say  $green \add red$ {\it evolves to} $green(red)$. Similarly for $green \adc red$.

In this paper, we present \colw\
which is a  web-based implementation of \propel. This implementation is  simple and
straightforward.
What is interesting is that \propelw\ is a novel and promising distributed (logic)
programming model with evolving knowledgebase.
It would
  provide a good starting point for future
 distributed logic programming as well as high-level web programming.

\section{Preliminaries}\label{s2}

In this section a  brief overview of CoL is given. 

There are two players: the machine $\pp$ and the environment $\oo$.

There are {\em elementary} atoms $p$, $q$, \ldots to represent elementary games.
\begin{description}
\item[Constant elementary games]  $\twg$ is always a true proposition, and $\tlg$ is always a false proposition.

\item[Negation]
 $\gneg$ is a role-switch operation: For example, $\gneg (0=1)$ is true,
while $(0=1)$ is false.

\item[Choice operations]
The choice group of operations:  $\adc$, $\add$ are defined below.

$A_0 \adc A_1$ is the game where, in the initial position, only $\oo$ has a legal move which consists in 
choosing $i$ in $\{0,1 \}$. After $\oo$ makes a move $i\in\{0,1 \}$, 
the game continues as $A_i$. 
 $\add$ is 
symmetric to $\adc$ with
the  difference that now it is $\pp$ who makes an initial move.

\item[Parallel operations]
Playing $A_1\mlc\ldots\mlc A_n$ means playing the $n$ games concurrently.  In order to win,  $\pp$ needs to win in each of $n$ games. Playing  $A_1\mld\ldots\mld A_n$ also means playing the $n$ games concurrently.  In order to win,  $\pp$ needs to win  one of the games. To indicate that a given move is made in the $i$th component, the player should prefix it with the string ``$i.$".

\item[Reduction]
 $A\mli B$ is defined  by $\gneg A\mld B$.
Intuitively, $A\mli B$ is the problem of reducing $B$ ({\em consequent}) to $A$ ({\em antecedent}).  

\end{description}

\section{General AI}\label{sec:intro}

In this section, 
we present a promising approach to reaching general AI.
Central to our approach is the concept of games\cite{Jap0, Japtow}.
This concept makes it possible to build intelligent AI
in a simplest possible way, as 
complex interactions among agents can be captured by games.
That is,  general AI is nothing but a group of agents playing games (or providing services to ) against others.

To be specific, we use the following idea: \\

   Agent = KB + Query \\

\noindent where an agent tries to solve Query using its knowledgebase
KB. Note that here KB and Q  both represent games and thus evolving.

 The following is a motivating example of \colw\ with agents $x,y,z,w,v,u,m,n$ and $o$.
   
  \begin{exmple}
 $x\ =\ p(3) \mlc \gneg p(100)$ \% $p(x)$  mean $x$ is prime. \\
 $y\ =\  (p(3)\add  p(5))^{ x }$ \\
$z\ =\  (p(4) \add p(5))^{ x }$ \\
$w\ =\ \gneg\ (p(9) \adc p(100))^x$ \\
$v\ =\ (\gneg p(9) \add \gneg p(100))^x$ \\
$m\ = \pp$, $[(p(0)\add p(3))^{y} \mli (p(0)\add p(3))^{u}]$ \\
$o\ = \pp$, $[(p(0)\add p(3))^{z} \mli (p(0)\add p(3))^{u}]$ \\
$n\ = \pp$, $[(p(100)\add \gneg p(100))^{w} \mli (p(100)\add 
\gneg p(100))^{u}]$

 \end{exmple}

\noindent 
 {\it Activating}  $y$  means $y$ is required to solve the incoming queries using the knowledgebase of $y$.

Now consider the machine $m$. It tries to solve the problem 
$(p(0)\add p(3))^{y} \mli (p(0)\add p(3))^{u}$ with
empty knowledgebase, denoted by $\pp$.
$m$ activates  $y$ which then tries to solve the goal $(p(0)\add p(3))^{m}$ using  $ (p(0)\add p(1))^{ x }$.
$y$ activates $x$ which then tries to solve the goal $(p(0)\add p(3))^{y}$  using $p(0)$.
This will succeed and $m$ eventually chooses  $p(3)^{u}$.  
Note that $o$ also wins the problem 
$(p(0)\add p(3))^{z} \mli (p(0)\add p(3))^{u}$,  as $z$ fails to make a move.

Similarly, consider the machine $n$. It tries to solve the problem 
$(p(100)\add \gneg p(100))^{w} \mli (p(100)\add \gneg p(100))^{u}$ with
empty knowledgebase.
$n$ activates  $w$ which then tries to solve the goal 
$(p(100)\add \gneg p(100))^{w}$ using  $\gneg (p(9)\adc p(100))^{ x }$.
$w$ activates $x$ which then tries to solve the goal 
$(\gneg p(9) \add \gneg p(100))^{x}$  using $\gneg p(100)$.
This will succeed by $x$ choosing $\gneg p(100)$.  $n$ eventually chooses  $(\gneg p(100))^{u}$.  
Note that  $v$ is logically equivalent to $w$.

\section{\propelw}\label{s2tb}

 We review the most basic fragment of propositional computability 
logic  called $\propel$ \cite{JapCL1}. Its language extends that of classical propositional logic by incorporating into it $\adc$ and $\add$. As always, there are infinitely many {\bf atoms} in the language, for which we will be using the letters
$p,q,r,\ldots$ as metavariables.  The two  atoms: $\twg$ and $\tlg$ have a special status in that their interpretation is fixed.  Formulas of this language, referred to as {\bf $\propel$-formulas}, are built from atoms in the standard way:

\begin{definition}
  
  The class of $\propel$-formulas 
is defined as the smallest set of expressions such that all atoms are in it and, if $F$ and $G$ are in it, then so are 
$\gneg F$, $F\mlc G$, $F \mld G$, $F \mli G$, $F\adc G$, $F \add G$. 
\end{definition}

\begin{definition}
  Let $F$ be a  $\propel$-formula. An interpretation is a function $^*$ which sends $F$ to a game $F^*$. $F$ is  said to be valid if, for every interpretation $^*$, there is a machine who wins the
  game $F^*$ for all possible scenarios corresponding to different behaviors by the environment.
 \end{definition}

Now we define $\propelw$, a slight extension to $\propel$ with environment parameters.
Let $F$ be a  $\propel$-formula.
We introduce a new {\it env-annotated} formula $F^\omega$ which reads as `play $F$ against an agent $\omega$'
or `provide a service $F$ to $\omega$'.
For an $\adc$-occurrence $O$ in $F^\omega$, we say
$\omega$ is the {\it matching} environment of $O$.
For example, $(p \adc (q \adc r))^{\omega}$ is an  agent-annotated formula and
$\omega$ is the matching environment of both occurrences of $\adc$.
We extend this definition to subformulas and formulas. For a subformula $F'$ of the above $F^\omega$,
we say that $\omega$ is the {\it matching} environment of both $F'$ and $F$.

In introducing environments to a formula $F$, one issue is
whether we allow `env-switching' formulas of the form
$(F[R^u])^w$. Here $F[R]$ represents a formula with some occurrence of a subformula $R$.
That is, the machine initially plays $F$ against agent $w$ and then switches
to play against another agent $u$ in the course of playing $F$. This kind of formulas are  difficult to
process. For this reason, in this paper, we focus on non `env-switching' formulas.
This leads to the following definition: 

\begin{definition} 
  The class of $\propelw$-formulas
is defined as the smallest set of expressions such that
(a) For any $\propel$-formula $F$ and any agent $\omega$, $F^\omega$  are in it and, (b) if $H$ and $J$ are in it, then so are 
$\gneg H$, $H\mlc J$, $H\mld J$, $H\mli J$.
\end{definition}
\noindent
In the above, $F^\omega$ denotes that the (current) machine provides a service $F$ to $\omega$.
$\gneg (F^\omega)$ denotes that the machine receives a service $F$ from $\omega$ ( i.e. the exchange of roles).
$F^\omega \mlc G^\mu$ denotes that the machine provides a service $F$ to $\omega$ and a service $G$ to
$\mu$. Similarly for $\mld, \mli$.

For example, suppose kim, pete are agents and
$p$ denotes a proposition.
Then,  $p^{kim} \mli p^{pete}$ denotes the following:
if $kim$ claims $p$ to the machine, then the machine can claim $p$ to pete.
This is clearly valid. 

We often use $F$ instead of $F^{\omega}$ when it is irrelevant.
For example, $p \mli p^{pete}$ denotes the following:
if some (unspecified) agent claims $p$ to the machine, then the machine can claim $p$ to pete.
Again, this is valid.

Most old concepts such as validity extend to this new language.

\begin{definition}
  Let $J$ be a  $\propelw$-formula. An interpretation is a function $^*$ which sends $F$ to a game $F^*$. $J$ is  said to be valid if, for every interpretation $^*$, there is a machine who wins the
  game $J^*$ for all possible scenarios corresponding to different behaviors by {\it any} environments.
 \end{definition}

\begin{definition}
\noindent Given a $\propelw$-formula $J$,  the skeleton  of $J$ -- denoted by
$skeleton(J)$ -- is obtained by
replacing every occurrence $F^\omega$ by $F$.
\end{definition}
\noindent For example, $skeleton((p \adc (q \adc r))^{\omega}) = p \adc (q \adc r)$.


We assume that each agent is identified with a physical location 
and the KB of an agent is stored in its location.

The following definitions comes from \cite{JapCL1}. They  apply both to $\propel$ and $\propelw$.

Understanding $E\mli F$ as an abbreviation of $\neg E \mld F$, a {\bf positive} occurrence of a subformula is one that is in the scope of an even number of $\neg$'s. Otherwise, the occurrence is {\bf negative}.

A {\bf surface occurrence} of a subformula means an occurrence that is not in the scope of a choice ($\add$ or $\adc$) operator.

A formula is  {\bf elementary} iff it does not contain the choice operators.

The {\bf elementarization} of a formula is the result of replacing, in it, every surface occurrence of the form $F_1\add ... \add F_n$ by $\oo$ , and every surface occurrence of the form $F_1\adc ... \adc F_n$ by $\pp$.

A formula is {\bf stable} iff its elementarization is valid in classical logic, otherwise
it is {\bf instable}.

$F${\bf -specification} of $O$, where  $F$ is a formula and $O$ is a surface occurrence in $F$, is a string $\alpha$ which can be defined by:
\begin{itemize}

\item $F$-specification of the occurrence in itself is the empty string.

\item If $F$ = $\neg G$, then $F$-specification of an occurrence that happens to be in $G$ is the same as the $G$-specification of that occurrence.

\item If $F$ is $G_1\mlc ... \mlc G_n$, $G_1\mld ... \mld G_n$, or $G_1\mli G_2$, then $F$-specification of an occurrence that happens to be in $G_i$ is the string $i.\alpha$, where $\alpha$ is the $G_i$-specification of that occurrence.

\end{itemize}

The proof system of \propelw\ is identical to that $\propel$ and
has the following two rules, with $H$, $F$ standing for $\propelw$-formulas and $\vec H$ for a set
of $\propelw$-formulas: \\

Rule (A): ${\vec H}\vdash F$, where $skeleton(F)$ is stable and, whenever $F$ has a positive (resp. negative) surface occurrence of $G_1\adc ... \adc G_n$ (resp. $G_1\add ... \add G_n$) whose matching environment is $\omega$, for each i$\in \{1,...,n\}$, $\vec H$ contains the result of replacing in $F$ that occurrence by $G_i^\omega$.

Rule (B): $H\vdash F$, where $H$ is the result of replacing in $F$ a negative (resp. positive) surface occurrence of $G_1\adc ... \adc G_n$ (resp. $G_1\add ... \add G_n$) whose matching environment is $\omega$ by $G_i^\omega$ for some i$\in \{1,...,n\}$.

\begin{examplee}\label{ex01}

$\propelw \vdash ((p\adc q)\mlc(p\adc q))\mli (p\adc q)^\omega$

  where $p$, $q$ represent distinct non-logical atoms, and $\omega$ is an agent.
  Note that $\omega$ plays no roles in the proof procedure.
\end{examplee}
\begin{enumerate}

\item $(p\mlc p)\mli p^\omega$, rule A, no premise

\item $(q\mlc q)\mli q^\omega$, rule A, no premise

\item $((q\adc p)\mlc p)\mli p^\omega$, rule B, 1

\item $((p\adc q)\mlc (q\adc p))\mli p^\omega$, rule B, 3

\item $((p\adc q)\mlc q)\mli q^\omega$, rule B, 2

\item $((p\adc q)\mlc (p\adc q))\mli q^\omega$, rule B, 5

\item $((p\adc q)\mlc (p\adc q))\mli (p\adc q)^\omega$, rule A, 4 6

\end{enumerate}

\begin{examplee}\label{ex02}
$\propelw \vdash p\mli (q\add p)^\omega$

where $p$, $q$ represent distinct non-logical atoms.
\end{examplee}
\begin{enumerate}

\item $p\mli p^\omega$, rule (A).  no premise

\item $p\mli (q\add p)^\omega$, rule B. 1

\end{enumerate}

\section{\colw}\label{s22tog}

In our setting, an agent   has  knowledgebase and receives
multiple queries.
\colw\ is a set of agent declarations of the following form: \\

\begin{exmple}
\>  $\alpha_1 = H_1, Q_1$ \\
\> \> $\vdots$ \\
 \>  $\alpha_n = H_n,Q_n$ \\
 \end{exmple}
  \noindent
  In the above, each
   $\alpha_i$ is an agent, each $H_i$ is the knowledgebase of $\alpha_i$ written in   $\propelw$ and
   each $Q_i$ is a queue for storing the incoming queries. We often omit $Q_i$ if it is initially empty.

  \subsection{An Execution Model for a  Query}\label{s22tog}
  
We first consider a machine model with empty knowledgebase and a single query to process.
This machine is designed to decide whether the query is valid or not.

The machine model of \propel\ is designed to play against any environment, and thus
easily extended to the case of \propelw.
In our system, however,  for each occurrence of $F^\omega$,
we need to differentiate $F$ which is already in session from those 
who are $not$.  That is, we  invoke $F$ to $\omega$ only when $F$ is not in session.
Below the notation $F[E]$ represents a formula $F$ together with some
positive occurrence of a subformula $E$.

 Below we will introduce an algorithm that executes
a formula $J$ which has a $\propelw$-proof. The algorithm contains two stages: \\

Algorithm $Ex(J)$: \%  $J$ is a $\propelw$-formula with a proof\\

\begin{enumerate}

\item First stage is to initialize a temporary variable $E$ to $J$,
activate all the resource agents specified in $J$ by invoking proper queries to them.
That is, 

\begin{itemize}

\item for each negative occurrence  $F[G_1\add G_2]^\omega$ in $J$ which is not
not already in session, 
    activate $\omega$ by querying $F^\mu$ to $\omega$. 
 Here $\mu$ is the current machine. Mark $F$ in session for
 $\omega$'s sake. 
 
\item for each positive occurrence  $F[G_1\adc G_2]^\omega$ in $J$ which is not
not already in session, we first replace it with 
$\gneg\ (\gneg F[G_1\adc G_2])^\omega$ and then
    activate $\omega$ by querying $(\gneg F[G_1\adc G_2])^\mu$ to $\omega$. 
 Here $\mu$ is the current machine; Mark $\gneg F$ in session for
 $\omega$'s sake.

\end{itemize}
        
\item The second stage is to play $J$ according to the following $loop$ procedure
(which is from \cite{JapCL1}): 

\end{enumerate}

procedure $loop(Tree)$: \%  $Tree$ is a proof tree of $J$ \\

{\bf Case} $E$ is derived by Rule (A): \\
    \hspace{3em}         Wait for the matching adversary $\omega$ to make a move  
    $\alpha =\beta i$, where $\beta$ \ $E$-specifies a  positive (negative) surface occurrence of a
    subformula $G_1\adc\ldots\adc G_n$ ($G_1\add\ldots\add G_n$) and $i\in\{1,\ldots,n\}$. 
  Let $H$ be  
the result of substituting in $E$ the above occurrence by $G_i$. Then update $E$ to $H$. \\

 {\bf Case} $E$ is derived by Rule (B): \\
  \hspace{3em}       Let $H$ be the premise of $E$ in the proof. $H$ is the result of substituting, in $E$, a certain negative (resp. positive) surface occurrence of a subformula $G_1\adc\ldots\adc G_n$ (resp. $G_1\add\ldots\add G_n$) by $G_i$ for some $i\in\{1,\ldots,n\}$. 
Let $\beta$ be the $E$-specification of that occurrence. Then make the move $\beta i$, update $E$ to $H$. Let $\omega$ be the matching environment. Then inform $\omega$ of the move  $\beta i$.

The following proposition has been proved in \cite{JapCL1}.

\begin{proposition}\label{sound}
  $\propel\vdash F$ iff $F$ is valid {\em (}any  $\propel$-formula $F${\em )}. 
\end{proposition}

The following  proposition follows easily from Proposition \ref{sound}, together with
the observation that $\propel$-proof of $F$ encodes an {\it environment-independent}
winning strategy for $F$.
The following is our  theorem  \cite{JapCL1}.

\begin{proposition}\label{sound2}
 Let $m$ be a machine above with empty knowledgebase and a $\propelw$-formula query $J$.
 Then the following holds:
\begin{enumerate}

  \item $\propelw\vdash J$ iff $J$ is valid. \\

  \item Furthermore, the following holds:
  \begin{itemize}
   \item  If  $\propelw\vdash J$, then  $m$    wins $J^*$ for every interpretation *.
   \item If $\propelw\vdash J$ does not hold, then $J^*$ is not computable for some interpretation.
   \end{itemize}
   \end{enumerate}
  
 \end{proposition}

\begin{proof}
Let $F$ be $skeleton(J)$.
   It is known from \cite{JapCL1} that
  every $\propelw$(/$\propel$)-proof of $J$ encodes an environment-independent winning strategy
  for $J$. It follows that a machine with such a strategy -- $Ex(J)$ -- wins $J$ against
  any environment. In particular, if  $J$ is stable, $\alpha = H$, $F_1^{\alpha}$ is in $J$ and $H\mli F_1$ does not have a proof,
 $\alpha$  does not make any moves  and $m$ wins because $J$ is stable.
   Hence $J$ is valid. Conversely, suppose there is no
  $\propelw$/$\propel$-proof of $J$.  Since $\propelw$-proof of $J$ is in fact
  identical to $\propel$-proof of $F$, it follows from \cite{JapCL1} that there is
  no machine who can win $F^*$ for some interpretation $*$. Therefore $F$ is not valid.
     \end{proof}

\subsection{Execution Model for Multiple Queries}

We now describe a machine model with nonempty knowledgebase and a sequence of queries to process.
It is designed to solve these  queries using its knowledgebase.

We assume that every agent 
processes multiple queries in a sequential fashion.
To do this, it maintains a queue  for storing multiple queries
$[ Q_1,\ldots,Q_n ].$
We assume that an agent $m = H_1$  processes $[ Q_1,\ldots,Q_n ]$ by executing the following $n$
 procedures {\it sequentially}:

\[ Ex(H_1\mli Q_1), Ex(H_2\mli Q_2),\ldots, Ex(H_n\mli Q_n) \]
\noindent
Here we assume that, for $i = \{ 1,\ldots,n-1 \}$, $H_i$ evolves to $H_{i+1}$ after performing $Ex(H_i\mli Q_i)$.

The following is a straightforward generalization of Proposition \ref{sound}.

\begin{proposition}\label{sound2}

Let $m$ be a machine with empty knowledgebase and incoming queries $[J_1,\ldots,J_n]$.
Then the following holds: \\
  
  \begin{enumerate}
  \item For all $i$, $\propelw\vdash J_i$ iff $J_i$ is valid {\em (}any $\propelw$-formula $J_i${\em )}. 
  \item Furthermore, the following holds:
  \begin{itemize}
   \item  If  $\propelw\vdash J_i$, then  $m$    wins $J_i^*$ for every interpretation *.
   \item If $\propelw\vdash J_i$ does not hold, then $J_i^*$ is not computable for some interpretation.
   \end{itemize}
     \end{enumerate}
  \end{proposition}
  
   Now we consider a machine with  nonempty knowledgebase.
An agent with nonempty knowledgebase processes queries in a way that it preserves soundness but not completeness.
For example, suppose $m$ has knowledgebase  $p\adc q$ with two queries $[ p\adc p, p\adc q ] $. 
Although both queries are
a logical consequence of $m$, solving the second query will fail. This is because $m = p\adc q$  would
evolve to $m = p$ after solving the first query.

\begin{proposition}\label{sound3}

Let $m$ be a machine with nonempty knowledgebase $H$ and incoming queries $[J_1,\ldots,J_n]$.
Assume $H$ evolve to $H_i$ after solving $J_1,\ldots,J_{i-1}$. Then the following holds:
  
\begin{enumerate}
\item $\propelw\vdash H_i\mli J_i$ iff $H_i\mli J_i$ is valid {\em (}any $\propelw$-formula $H_i,J_i${\em )}. 
Furthermore,  \\

 if $\propelw\vdash H_i\mli J_i$, then   $m$ wins $(H_i\mli J_i)^*$ for every interpretation *.   \\
\item If  $\propelw\vdash H_i\mli J_i$, then $\propelw\vdash H\mli J_i$.
\item (Soundness:) If $m$ wins $(H_i\mli J_i)^*$ for every interpretation *, then $H\mli J_i$  is valid. 
\end{enumerate}
\end{proposition}

\begin{proof}
Proof of (1): It is an easy consequence of  Proposition \ref{sound}.

Proof of (2): It is easy to  observe that if $H$ evolves to $H_i$, then $H_i$
  is a logical consequence of $H$. Hence, $J_i$ is a logical consequence of $H$.
  
Proof of (3):    If $m$ successfully solves $H_i\mli J_i$, then it follows from (1) that 
 $H_i\mli J_i$ has a proof. Then it follows from (2) that $H\mli J_i$ has a proof.
 It follows from Proposition \ref{sound} that $H\mli J_i$ is valid.
\end{proof}

\section{Examples}
 
One example  is provided by the 
 following ``weather'' agent which contains today's weather (we assume today is cloudy) and
 temperature (we assume today is hot).

\begin{exmple}
\> $weather = cloudy \mlc hot$.\\
\end{exmple}

Our language  permits  `querying knowledge' of the form
$Q^\omega$ in KB. This requires the current machine to
invoke the query $Q$ to the agent $\omega$. 
Now let us consider the $dress$ agent which gives
advice on the dress codes according to the weather condition.
It contains the following four rules and 
two querying  knowledges  $(cloudy \add  sunny)$ and $(hot \add  cold)$ relative
to the $weather$ agent.

\begin{exmple}
\> $dress =$\\
\> \% dress codes \\
\>$(cloudy \mlc hot)\mli green$. \\
\>$(sunny \mlc hot) \mli yellow$.  \\
\>$(cloudy \mlc cold)\mli blue$. \\
\>$(sunny \mlc cold) \mli red$.  \\
\> $(cloudy \add  sunny)^{weather}$\\
\> $(hot \add  cold)^{weather}$.\\
\end{exmple}

Now, consider a machine $m$ trying to solve the query 
?- $ (green \add blue \add yellow \add red)^{dress} \mli (green \add blue \add yellow \add red)^{user} $ with respect to
empty knowledgebase.
This is written as
\begin{exmple}
\> $m = \pp$, \\
\> $[ (green \add blue \add yellow \add red)^{dress} \mli  (green \add blue \add yellow \add red)^{user} ].$
\end{exmple}

Solving this goal  has the effect of activating $dress$ and invoking two queries
$(cloudy \add  sunny)$  and $(hot \add cold)$ to the $weather$ agent.
At this stage, the $dress$ and $weather$ agents 
remain active and communicate with each other.
To be specific,  the $weather$ solves these two queries
using $\propelw$ proof and the $Ex$ procedure in the previous section.
  This would result in replacing 
 $(cloudy \add  sunny)$  with $cloudy$ and
$(hot \add cold)$ with $hot$. Now the $dress$ agent -- again via  the 
$\propelw$ proof and the $Ex$ procedure -- will   answer
 $green^{\pp}$ to the machine. The machine chooses $green^{user}$ and informs the user. Note that two queries to $weather$ execute concurrently within $weather$.

\section{Conclusion} \label{s5thr}

In this paper, we proposed  a multi-agent programming model based on $\propel$.
Unlike other formalisms such as LogicWeb\cite{Loke} and distributed logic programming\cite{LCF},
this model supports evolving knowledgebase which is essential for future computing model.
Our next goal is to replace $\propel$ with much more expressive $\propeltw$\cite{Japtow}.

\section{Acknowledgements}

We thank Giorgi Japaridze for many helpful comments.

\bibliographystyle{plain}

\end{document}

%% file: clmacros1.tex
\newcommand{\code}[1]{\ulcorner #1 \urcorner}
\newcommand{\mldi}{\hspace{2pt}\mbox{\footnotesize $\vee$}\hspace{2pt}}
\newcommand{\mlci}{\hspace{2pt}\mbox{\footnotesize $\wedge$}\hspace{2pt}}
\newcommand{\emptyrun}{\langle\rangle} 
\newcommand{\oo}{\bot}            
\newcommand{\pp}{\top}            
\newcommand{\xx}{\wp}               
\newcommand{\legal}[2]{\mbox{\bf Lr}^{#1}_{#2}} 
\newcommand{\win}[2]{\mbox{\bf Wn}^{#1}_{#2}} 
 \newcommand{\one}{\mbox{\sc One}}
 \newcommand{\two}{\mbox{\sc Two}}
 \newcommand{\three}{\mbox{\sc Three}}
 \newcommand{\four}{\mbox{\sc Four}}
 \newcommand{\first}{\mbox{\sc Derivation}}
 \newcommand{\second}{\mbox{\sc Second}}
 \newcommand{\uorigin}{\mbox{\sc Org}}
 \newcommand{\image}{\mbox{\sc Img}}
 \newcommand{\limitset}{\mbox{\sc Lim}}
 \newcommand{\fif}{\mbox{\bf CL15}}
\newcommand{\col}[1]{\mbox{$#1$:}}

\newcommand{\sti}{\mbox{\raisebox{-0.02cm}
{\scriptsize $\circ$}\hspace{-0.121cm}\raisebox{0.08cm}{\tiny $.$}\hspace{-0.079cm}\raisebox{0.10cm}
{\tiny $.$}\hspace{-0.079cm}\raisebox{0.12cm}{\tiny $.$}\hspace{-0.085cm}\raisebox{0.14cm}
{\tiny $.$}\hspace{-0.079cm}\raisebox{0.16cm}{\tiny $.$}\hspace{1pt}}}
\newcommand{\costi}{\mbox{\raisebox{0.08cm}
{\scriptsize $\circ$}\hspace{-0.121cm}\raisebox{-0.01cm}{\tiny $.$}\hspace{-0.079cm}\raisebox{0.01cm}
{\tiny $.$}\hspace{-0.079cm}\raisebox{0.03cm}{\tiny $.$}\hspace{-0.085cm}\raisebox{0.05cm}
{\tiny $.$}\hspace{-0.079cm}\raisebox{0.07cm}{\tiny $.$}\hspace{1pt}}}

\newcommand{\seq}[1]{\langle #1 \rangle}           


\newcommand{\mla}{\mbox{{\Large $\wedge$}}}
\newcommand{\mle}{\mbox{{\Large $\vee$}}}

\newcommand{\pst}{\mbox{\raisebox{-0.01cm}{\scriptsize $\wedge$}\hspace{-4pt}\raisebox{0.16cm}{\tiny $\mid$}\hspace{2pt}}}
\newcommand{\gneg}{\neg}                  
\newcommand{\mli}{\rightarrow}                     
\newcommand{\cla}{\mbox{\large $\forall$}}      
\newcommand{\cle}{\mbox{\large $\exists$}}        
\newcommand{\mld}{\vee}    
\newcommand{\mlc}{\wedge}  
\newcommand{\ade}{\mbox{\Large $\sqcup$}}      
\newcommand{\ada}{\mbox{\Large $\sqcap$}}      
\newcommand{\add}{\sqcup}                      
\newcommand{\adc}{\sqcap}                      

\newcommand{\tlg}{\bot}               
\newcommand{\twg}{\top}               
\newcommand{\st}{\mbox{\raisebox{-0.05cm}{$\circ$}\hspace{-0.13cm}\raisebox{0.16cm}{\tiny $\mid$}\hspace{2pt}}}
\newcommand{\cst}{{\mbox{\raisebox{-0.05cm}{$\circ$}\hspace{-0.13cm}\raisebox{0.16cm}{\tiny $\mid$}\hspace{1pt}}}^{\aleph_0}} 
\newcommand{\cost}{\mbox{\raisebox{0.12cm}{$\circ$}\hspace{-0.13cm}\raisebox{0.02cm}{\tiny $\mid$}\hspace{2pt}}}
\newcommand{\ccost}{{\mbox{\raisebox{0.12cm}{$\circ$}\hspace{-0.13cm}\raisebox{0.02cm}{\tiny $\mid$}\hspace{1pt}}}^{\aleph_0}} 
\newcommand{\pcost}{\mbox{\raisebox{0.12cm}{\scriptsize $\vee$}\hspace{-4pt}\raisebox{0.02cm}{\tiny $\mid$}\hspace{2pt}}}

\newcommand{\intimpl}{\mbox{\hspace{2pt}$\circ$\hspace{-0.14cm} \raisebox{-0.043cm}{\Large --}\hspace{2pt}}}
\newcommand{\sintimpl}{\mbox{\hspace{2pt}\raisebox{0.033cm}{\tiny $ | \hspace{-4pt} >$}\hspace{-0.14cm} \raisebox{-0.039cm}{\large --}\hspace{2pt}}}
\newcommand{\sst}{\mbox{\raisebox{-0.07cm}{\scriptsize $-$}\hspace{-0.2cm}$\pst$}}
\newcommand{\scost}{\mbox{\raisebox{0.20cm}{\scriptsize $-$}\hspace{-0.2cm}$\pcost$}}
\newcommand{\sqc}{\mbox{\hspace{2pt}\small \raisebox{0.0cm}{$\bigtriangleup$}\hspace{2pt}}}
\newcommand{\sqci}{\mbox{\scriptsize \raisebox{0.0cm}{$\bigtriangleup$}}}
\newcommand{\sqd}{\mbox{\hspace{2pt}\small \raisebox{0.06cm}{$\bigtriangledown$}\hspace{2pt}}}
\newcommand{\sqdi}{\mbox{\scriptsize \raisebox{0.05cm}{$\bigtriangledown$}}}
\newcommand{\sqe}{\mbox{\large \raisebox{0.07cm}{$\bigtriangledown$}}}
\newcommand{\sqa}{\mbox{\large \raisebox{0.0cm}{$\bigtriangleup$}}}
\newcommand{\tgd}{\mbox{\hspace{2pt}$\vee$\hspace{-1.29mm}\raisebox{0.1mm}{\rule{0.13mm}{2mm}}\hspace{5pt}}}    
\newcommand{\tgc}{\mbox{\hspace{2pt}$\wedge$\hspace{-1.29mm}\raisebox{0.02mm}{\rule{0.13mm}{2mm}}\hspace{5pt}}}    
\newcommand{\tge}{\hspace{1pt}\mbox{\Large $\vee$\hspace{-1.84mm}\raisebox{0.1mm}{\rule{0.13mm}{3.0mm}}\hspace{6pt}}}   
\newcommand{\tga}{\mbox{\hspace{1pt}\Large $\wedge$\hspace{-1.84mm}\raisebox{0.02mm}{\rule{0.13mm}{3.0mm}}\hspace{6pt}}}     
\newcommand{\tgpst}{\mbox{\raisebox{-0.01cm}{\scriptsize $\wedge$}\hspace{-4pt}\raisebox{0.06cm}{\small $\mid$}\hspace{2pt}}}
\newcommand{\tgpcost}{\mbox{\raisebox{0.12cm}{\scriptsize $\vee$}\hspace{-3.8pt}\raisebox{0.04cm}{\small $\mid$}\hspace{2pt}}}
\newcommand{\tgst}{\mbox{\raisebox{-0.05cm}{$\circ$}\hspace{-0.12cm}\raisebox{0.05cm}{\small $\mid$}\hspace{2pt}}} 
\newcommand{\tgcost}{\mbox{\raisebox{0.12cm}{$\circ$}\hspace{-0.12cm}\raisebox{0.04cm}{\small $\mid$}\hspace{2pt}}}


\newtheorem{theoremm}{Theorem}[section]
\newtheorem{conditionss}{Condition}[section]
\newtheorem{thesiss}[theoremm]{Thesis}
\newtheorem{definitionn}[theoremm]{Definition}
\newtheorem{lemmaa}[theoremm]{Lemma}
\newtheorem{notationn}[theoremm]{Notation}
\newtheorem{propositionn}[theoremm]{Proposition}
\newtheorem{conventionn}[theoremm]{Convention}
\newtheorem{examplee}[theoremm]{Example}
\newtheorem{remarkk}[theoremm]{Remark}
\newtheorem{factt}[theoremm]{Fact}
\newtheorem{exercisee}[theoremm]{Exercise}
\newtheorem{questionn}[theoremm]{Open Problem}
\newtheorem{conjecturee}[theoremm]{Conjecture}

\newenvironment{exercise}{\begin{exercisee} \em}{ \end{exercisee}}
\newenvironment{definition}{\begin{definitionn} \em}{ \end{definitionn}}
\newenvironment{theorem}{\begin{theoremm}}{\end{theoremm}}
\newenvironment{lemma}{\begin{lemmaa}}{\end{lemmaa}}
\newenvironment{proposition}{\begin{propositionn} }{\end{propositionn}}
\newenvironment{convention}{\begin{conventionn} \em}{\end{conventionn}}
\newenvironment{remark}{\begin{remarkk} \em}{\end{remarkk}}
\newenvironment{proof}{ {\bf Proof.} }{\  \rule{2.5mm}{2.5mm} \vspace{.2in} }
\newenvironment{idea}{ {\bf Idea.} }{\  \rule{1.5mm}{1.5mm} \vspace{.15in} }
\newenvironment{example}{\begin{examplee} \em}{\end{examplee}}
\newenvironment{fact}{\begin{factt}}{\end{factt}}
\newenvironment{notation}{\begin{notationn} \em}{\end{notationn}}
\newenvironment{conditions}{\begin{conditionss} \em}{\end{conditionss}}
\newenvironment{question}{\begin{questionn}}{\end{questionn}}
\newenvironment{conjecture}{\begin{conjecturee}}{\end{conjecturee}}

%% file: rootrml.bbl
\begin{thebibliography}{1}


\bibitem{LCF} E.S. Lam and I. Cervesato and N. Fatima.
  Comingle: Distributed Logic Programming for Decentralized Mobile Ensembles.
  {\it LNCS} 9037. 2015.

\bibitem{Jap0} Japaridze G. Introduction to computability logic. {\it Annals of Pure and Applied Logic}, 2003, 123(1/3): 1-99.


\bibitem{JapCL1} Japaridze G. Propositional computability logic I. {\it ACM Transactions on Computational Logic}, 2006, 7(2): 302-330.

\bibitem{Japtow} Japaridze G. Towards applied theories based on computability logic. {\it Journal of Symbolic Logic}, 2010, 75(2): 565-601.



\bibitem{JapCL12} Japaridze G. On the system CL12 of computability logic. http://arxiv.org/abs/1203.0103, June 2013.


\bibitem{Kwon} Kwon K, Hur S. Adding Sequential Conjunctions to Prolog. {\it International Journal of Computer Technology and Applications}, 2010, 1(1): 1-3.

\bibitem{Loke}
S.W. Loke and A. Davison:
LogicWeb: Enhancing the Web with Logic Programming. {\it Journal of Logic Programming},
1998, 36(3): 195-240.


\end{thebibliography}
